\documentclass[12pt]{article}
\usepackage{amsmath}
\usepackage{subcaption}
\usepackage{booktabs}
\usepackage{tabularx}
\usepackage{mathrsfs}
\usepackage{amsthm}
\usepackage{amssymb}
\usepackage{multirow}
\usepackage{makecell}
\usepackage{hyperref}
\usepackage{nicefrac}
\usepackage{siunitx}
\usepackage{rotating}
\usepackage{multirow}
\usepackage{makecell}

\usepackage[a4paper, margin=1in]{geometry}

\usepackage{biblatex} 
\usepackage{graphicx}
\usepackage{forest}
\usepackage{tikz-qtree}
\usepackage{authblk}

\usetikzlibrary{shapes.geometric}
\providecommand{\keywords}[1]
{
	\small
	\textbf{\textit{Keywords---}} #1
}
\theoremstyle{definition}
\newtheorem{definition}{Definition}
\newtheorem{theorem}{Theorem}[section]

\title{rKAN: Rational Kolmogorov-Arnold Networks}

\author{Alireza Afzal Aghaei}
\affil{Independent Researcher\\
Email: alirezaafzalaghaei@gmail.com}

\begin{document}
\maketitle

\small

\begin{abstract}
The development of Kolmogorov-Arnold networks (KANs) marks a significant shift from traditional multi-layer perceptrons in deep learning. Initially, KANs employed B-spline curves as their primary basis function, but their inherent complexity posed implementation challenges. Consequently, researchers have explored alternative basis functions such as Wavelets, Polynomials, and Fractional functions. In this research, we explore the use of rational functions as a novel basis function for KANs. We propose two different approaches based on Padé approximation and rational Jacobi functions as trainable basis functions, establishing the rational KAN (rKAN). We then evaluate rKAN’s performance in various deep learning and physics-informed tasks to demonstrate its practicality and effectiveness in function approximation.
\end{abstract}

\section{Introduction}
Function approximation is a crucial area of study within numerical analysis and computational mathematics. It involves using simpler functions, known as basis functions, to represent complex ones, thereby simplifying analysis and computation. This process is essential for various applications such as solving differential equations, data fitting, and machine learning \cite{gautschi2011numerical, alpaydin2020introduction}. By approximating functions, we can predict outcomes, optimize processes, and identify patterns in data. Various basis functions are employed for function approximation, each with its unique advantages suited to specific problem requirements. These functions serve as the building blocks for approximating more complex functions and can significantly influence the accuracy and efficiency of the approximation.

In numerical analysis, one common method is polynomial curve fitting, where the basis functions are polynomials. These functions are simple and easy to compute but can suffer from instability issues, especially with higher-degree polynomials. This phenomenon, known as Runge's phenomenon \cite{boyd1992defeating}, highlights the limitations of polynomial basis functions for certain types of data. Spline interpolation is another widely used method, particularly effective for functions with intricate shapes. Splines are piecewise polynomials that ensure smoothness at the points where the polynomial pieces connect, called knots. This method offers great flexibility and smoothness, making it ideal for applications requiring a high degree of accuracy in the approximation of curves and surfaces. 

Fourier series, which use trigonometric functions as basis functions, are particularly effective for approximating periodic functions. The Fourier basis functions, consisting of sines and cosines, can represent periodic behavior accurately and are widely used in signal processing, image analysis, and other fields requiring periodic function analysis \cite{boyd2001chebyshev}. Wavelets are another class of basis functions that have gained popularity, especially in signal and image processing. Wavelets enable multi-resolution analysis, providing a method to analyze data at various levels of detail. This is particularly beneficial for applications involving hierarchical or time-frequency analysis \cite{bozorgasl2024wav}. Fractional basis functions are another type of functions that can capture intricate behaviors and subtle variations in data that integer-order methods might miss. This makes them particularly useful in modeling natural phenomena \cite{babaei2024solving}. They are able to provide smooth approximations with fewer terms compared to integer-order polynomials, resulting in more efficient computations and a reduced risk of overfitting. Fractional B-splines, for example, offer greater flexibility in controlling the smoothness and continuity of the approximating function, making them ideal for applications requiring high precision and adaptability \cite{masti2024collocation}.

Rational approximations, where basis functions are ratios of polynomials, provide a robust method for approximating functions with asymptotic behavior and singularities. These basis functions are particularly useful for functions with sharp peaks or rapid changes. Padé approximants, a specific type of rational approximation, are known for their ability to achieve accurate approximations over a broad range of values. This method is particularly effective in areas such as control theory and complex analysis, where premasti2024collocationcise function approximation is critical \cite{babaei2024solving, boyd2001chebyshev, baker1961pade}.

Most of these basis functions have been developed for machine learning and deep learning tasks. In these fields, algorithms aim to approximate a function through a potentially nested combination of basis functions that best fit the given data. Examples include support vector machines with polynomial \cite{alpaydin2020introduction} or fractional kernels \cite{rad2023learning}, the least-squares support vector machines with an orthogonal rational kernel \cite{babaei2024solving}, and neural networks with various activation functions such as orthogonal Legendre \cite{parand2023neural, aghaei2023solving}, Fourier \cite{silvescu1999fourier, pratt2017fcnn}, fractional \cite{hadian2020single}, and rational functions \cite{boulle2020rational, siu1994rational, leung1993rational, telgarsky2017neural}. Additionally, B-spline neural networks have been explored for their modeling capabilities \cite{hong2011modeling, dos2009nonlinear, chen2014complex, bohra2020learning}. In some scenarios, B-spline neural networks can be regarded as Kolmogorov-Arnold neural networks.

Kolmogorov-Arnold Networks, based on the Kolmogorov-Arnold representation theorem \cite{liu2024kan, kiamari2024gkan, samadi2024smooth, liu2024initial}, offer a novel approach to accurately fitting real-world data. These networks have been applied to various domains, including time-series analysis \cite{genet2024tkan, vaca2024kolmogorov, xu2024kolmogorov}, human activity recognition \cite{liu2024ikan}, seizure detection \cite{herbozo2024kan}, electrohydrodynamic pumps \cite{peng2024predictive}, and cognitive diagnosis \cite{yang2024endowing}. Initially, these networks were developed using B-spline curves \cite{liu2024kan}. However, due to implementation challenges and issues with smoothness, alternative basis functions have been explored. These alternatives include Wavelet KANs \cite{bozorgasl2024wav, seydi2024unveiling}, Fourier KANs \cite{xu2024fourierkan}, radial basis function KANs \cite{li2024kolmogorov, ta2024bsrbfkan}, polynomial KANs \cite{seydi2024exploring, ss2024chebyshev}, and fractional KANs \cite{aghaei2024fkan}.

While some attempts have applied rational functions to traditional neural networks \cite{boulle2020rational,siu1994rational,leung1993rational,telgarsky2017neural}, there has been no research on the applicability of rational functions in KANs. In this paper, we examine the accuracy of KANs using rational functions through two different approaches: 1) Padé approximation and 2) a mapped version of Jacobi polynomials. The first approach uses the original Jacobi polynomials to construct a rational approximation in the Padé scheme for describing the data. The second approach maps the original Jacobi polynomial, defined on a finite interval, into a possibly infinite domain with a rational mapping. We demonstrate in both cases how the fractional KAN \cite{aghaei2024fkan} approach can be utilized as a generalized case in our methodology. Finally, we compare the results of these two approaches with KANs and other alternatives in deep learning tasks such as regression and classification. We also assess the accuracy of this approach in physics-informed deep learning tasks, particularly by approximating the solution of certain differential equations.

The rest of the paper is organized as follows. In Section 2, we review some preliminaries on Jacobi polynomials and their properties. Section 3 explains the KAN formulations and the proposed methodology. In Section 4, we validate the proposed method on several real-world problems. Finally, in Section 5, we present some concluding remarks.

\section{Jacobi polynomials}

Jacobi polynomials (denoted by $\mathcal{J}^{(\alpha,\beta)}_{n}(\xi)$) are an infinite sequence of orthogonal functions that are mutually orthogonal to each other \cite{rad2023learning}. Mathematically, the following inner product will be zero for $n \ne m$:
\begin{equation*}
    \langle \mathcal{J}_{m}^{(\alpha ,\beta )}, \mathcal{J}_{n}^{(\alpha ,\beta )}\rangle_{\omega(\xi)} = \int_{-1}^{1} 
\mathcal{J}_{m}^{(\alpha ,\beta )}(\xi) \mathcal{J}_{n}^{(\alpha ,\beta )}(\xi)  \omega(\xi) \text{d}\xi =  \langle \mathcal{J}_{n}^{(\alpha ,\beta )}, \mathcal{J}_{n}^{(\alpha ,\beta )}\rangle_{\omega(\xi)} \delta_{m,n}.
\end{equation*}
For $n=0,1,\ldots,$ these polynomials are defined by the Gamma function:
\begin{equation*}
    {\displaystyle \mathcal{J}_{n}^{(\alpha ,\beta )}(\xi)={\frac {\Gamma (\alpha +n+1)}{n!\,\Gamma (\alpha +\beta +n+1)}}\sum _{m=0}^{n}{n \choose m}{\frac {\Gamma (\alpha +\beta +n+m+1)}{\Gamma (\alpha +m+1)}}\left({\frac {\xi-1}{2}}\right)^{m}}.
\end{equation*}
In this definition, $\alpha, \beta > -1$ play the role of hyperparameters that affect the shape of the resulting function. We can treat these parameters as unknown weights in the computational graph and optimize them during the network's optimization process. However, we must ensure the validity of their values. For this purpose, we utilize the well-known ELU activation function \cite{clevert2015fast}, which possesses the property $\text{ELU}: \mathbb{R} \to (-\kappa, \infty)$:
\begin{equation*}
    \text{ELU}(\xi; \kappa) = \begin{cases} \xi & \text{if } \xi > 0, \\ \kappa \times (e^\xi - 1) & \text{if } \xi \leq 0, \end{cases}
\end{equation*}
where $\kappa$ is a parameter that controls the lower bound of the range of the ELU function. Consequently, to ensure meaningful Jacobi functions for parameters $\alpha$ and $\beta$, $\kappa$ can be set to $1$.

The Jacobi function is traditionally defined on the interval $[-1,1]$, which limits its use in approximating functions across desired intervals. Consequently, researchers have developed techniques to extend their definition to a potentially infinite domain. These extended functions can be generated by the following definition.

\begin{definition}[Mapped Jacobi function]
By applying an invertible mapping function $\varphi: \Omega \to [-1,1]$ to the input of Jacobi polynomials, the mapped Jacobi functions can be generated as:
\begin{equation*}
\mathcal{R}^{(\alpha,\beta)}_n(\xi) = \mathcal{J}^{(\alpha,\beta)}_n(\varphi(\xi)).
    \end{equation*}
\end{definition}
The choice of $\varphi(\cdot)$ can vary depending on the original problem domain. For instance, for a finite domain $\Omega=[d_0, d_1]$, the linear mapping
\begin{equation*}
       \varphi(\xi; d_0,d_1) = \frac{2\xi-d_0-d_1}{d_1-d_0},
   \end{equation*}
can be employed (Figure \ref{fig:shift-jacobi}). For a semi-infinite domain $\Omega=(0, \infty)$, three major options with the hyperparameter $\iota > 0$ are available:
    \begin{itemize}
        \item Logarithmic mapping (Figure \ref{fig:log-semi-jacobi}):  
    \begin{equation}
            \phi(\xi; \iota)=2\tanh\left(\dfrac{\xi}{\iota}\right)-1, \label{eq:map_log}
        \end{equation}
        \item Algebraic mapping (Figure \ref{fig:alg-semi-jacobi}): 
        \begin{equation}
            \phi(\xi; \iota)=\dfrac{\xi-\iota}{\xi+\iota}, \label{eq:map_alg}
        \end{equation}
        \item Exponential mapping (Figure \ref{fig:exp-semi-jacobi}):  
        \begin{equation}
            \phi(\xi; \iota)=1-2{\exp({-\displaystyle\dfrac{\xi}{\iota}})}. \label{eq:map_exp}
        \end{equation}
    \end{itemize}
Finally, for the infinite interval $\Omega=(-\infty, \infty)$, one can use a nonlinear mapping with $\iota > 0$ to generate mapped Jacobi functions:
    \begin{itemize}
        \item Logarithmic mapping (Figure \ref{fig:log-inf-jacobi}):  
    \begin{equation}
            \phi(\xi; \iota)=\tanh\left(\dfrac{\xi}{\iota}\right), \label{eq:map_log2}
        \end{equation}
        \item Algebraic mapping (Figure \ref{fig:alg-inf-jacobi}): 
        \begin{equation}
            \phi(\xi; \iota)=\dfrac{\xi}{\sqrt{\xi^2+\iota^2}}.\label{eq:map_alg2}
        \end{equation}
        
        \end{itemize}
When the domain $\Omega$ is semi-infinite or infinite, it is common to call $\mathcal{R}^{(\alpha,\beta)}_n(\xi)$ as rational Jacobi functions \cite{boyd2001chebyshev}. As illustrated in Figure \ref{fig:jacobi-plots}, these functions are non-zero and possess real-valued distinct roots within their domain. They are differentiable, and their derivatives can be expressed in terms of the functions themselves. For a more detailed discussion of these functions, we refer the reader to \cite{aghaei2023solving, rad2023learning, boyd2001chebyshev}.

\begin{figure}[htbp]
  \centering
  \begin{subfigure}{0.37\textwidth}
    \centering
    \includegraphics[width=\textwidth]{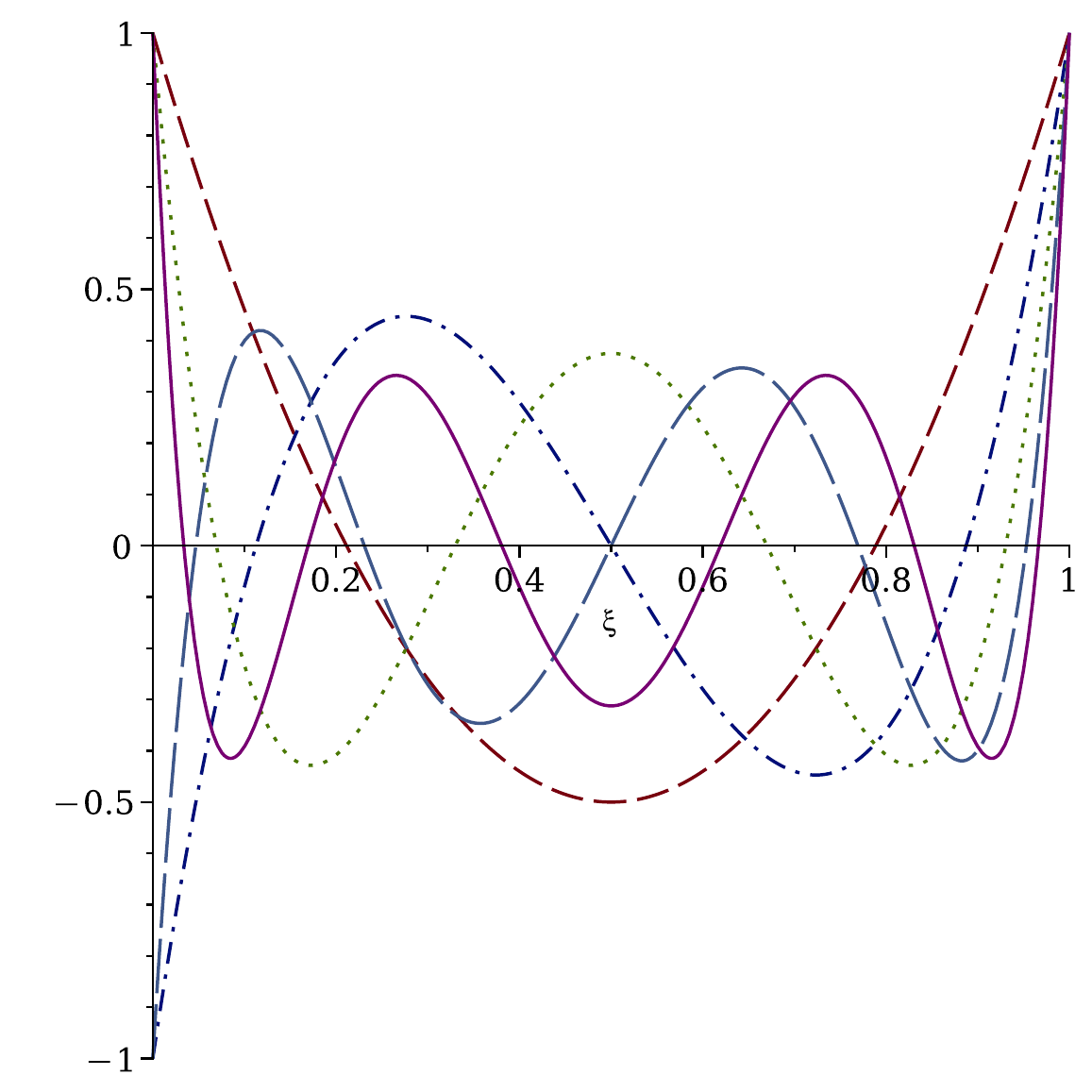}
    \caption{Finite mapping}
    \label{fig:shift-jacobi}
  \end{subfigure}%
  \begin{subfigure}{0.37\textwidth}
    \centering
    \includegraphics[width=\textwidth]{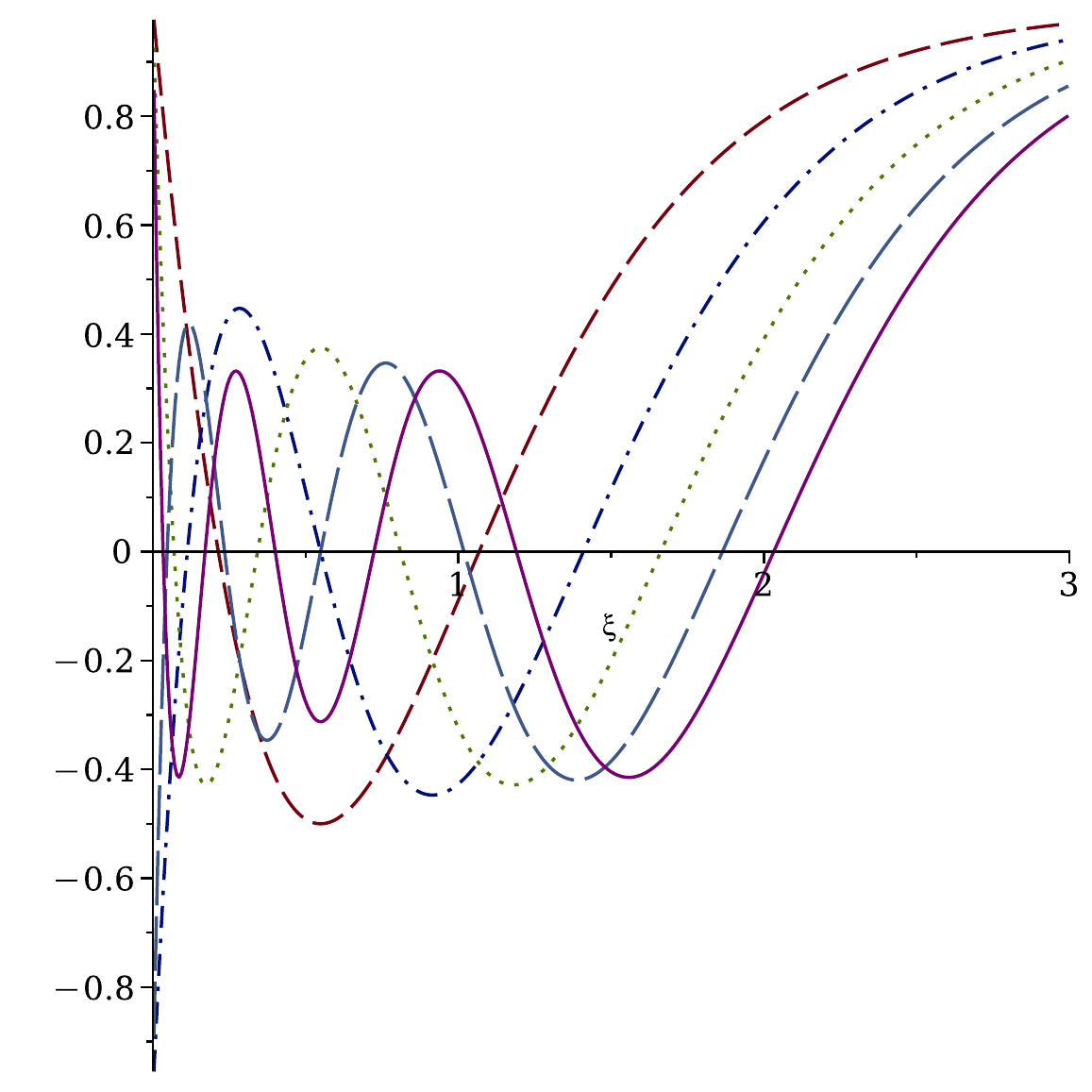}
    \caption{Logarithmic mapping $(0, \infty)$}
    \label{fig:log-semi-jacobi}
  \end{subfigure}
  \\
  \begin{subfigure}{0.37\textwidth}
    \centering
    \includegraphics[width=\textwidth]{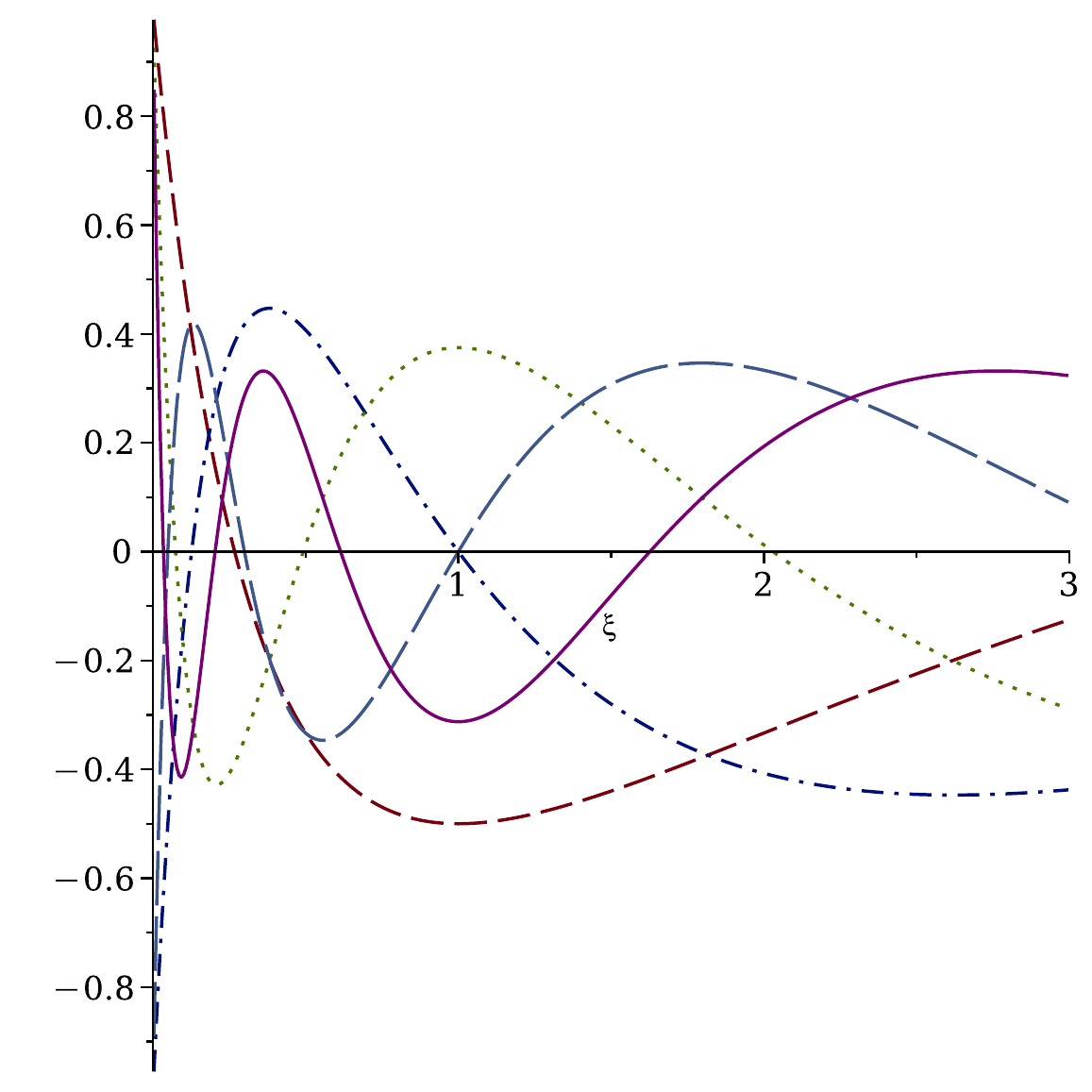}
    \caption{Algebraic mapping $(0, \infty)$}
    \label{fig:alg-semi-jacobi}
  \end{subfigure}%
  \begin{subfigure}{0.37\textwidth}
    \centering
    \includegraphics[width=\textwidth]{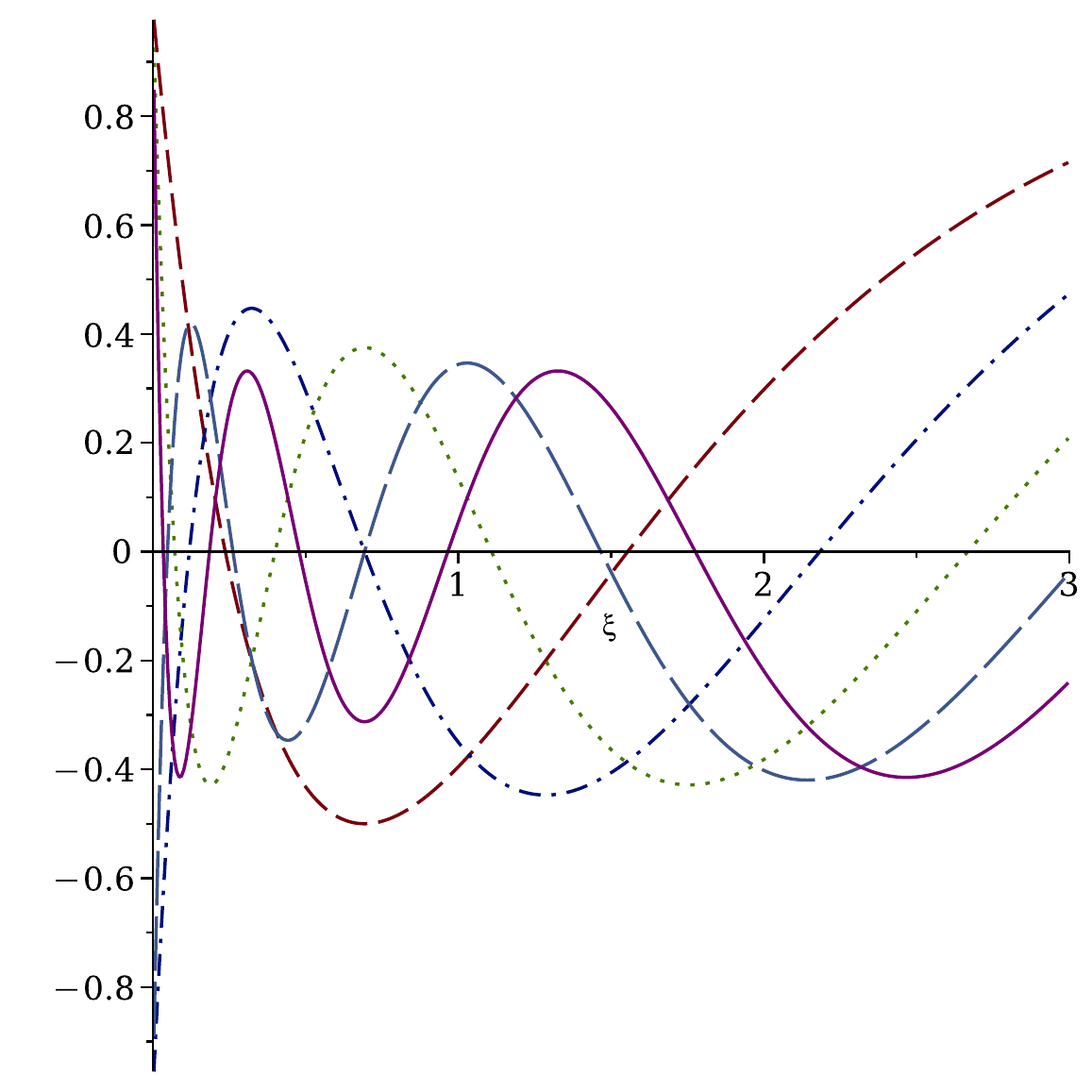}
    \caption{Exponential mapping $(0, \infty)$}
    \label{fig:exp-semi-jacobi}
  \end{subfigure}
  \\
  \begin{subfigure}{0.37\textwidth}
    \centering
    \includegraphics[width=\textwidth]{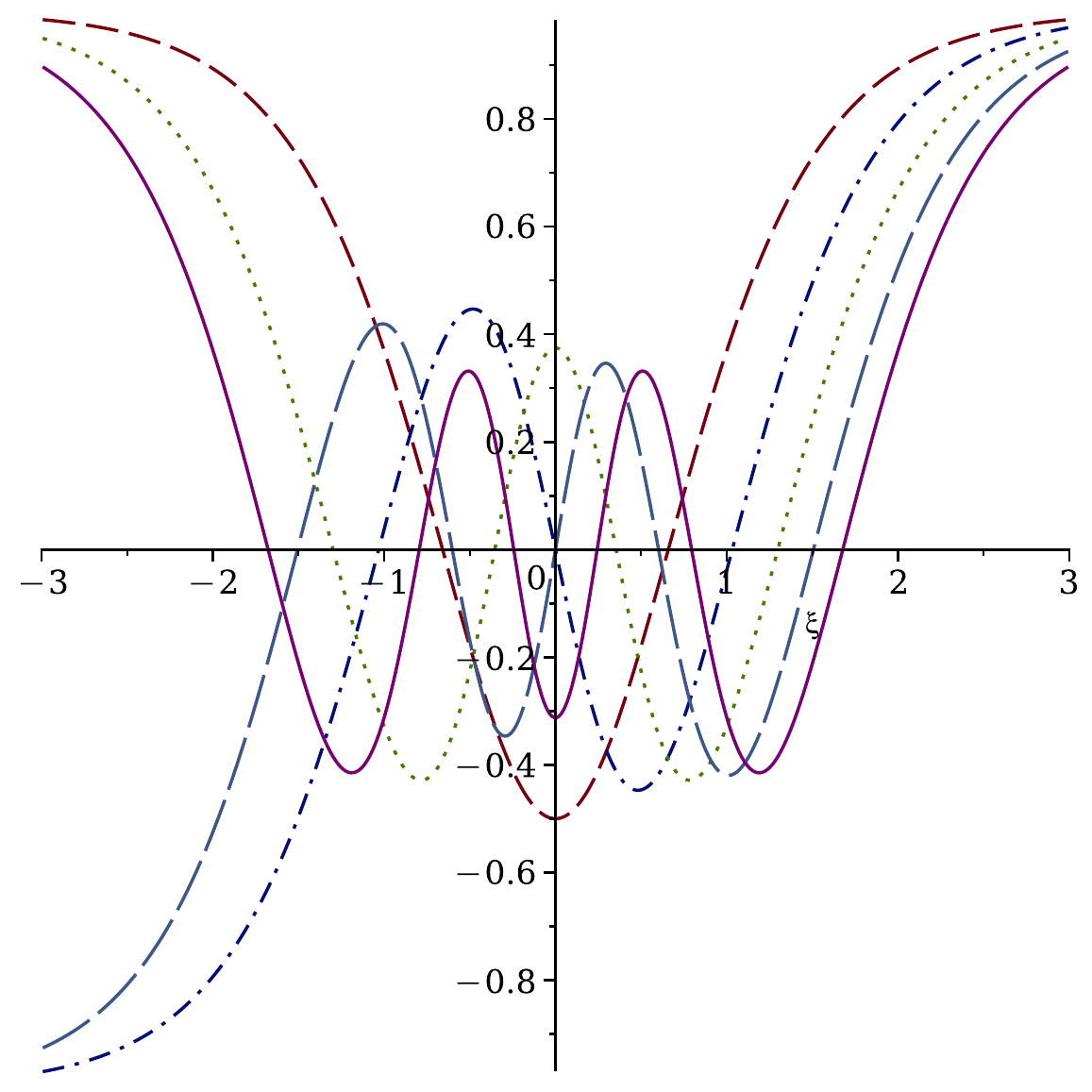}
    \caption{Logarithmic mapping $(-\infty, \infty)$}
    \label{fig:log-inf-jacobi}
  \end{subfigure}%
  \begin{subfigure}{0.37\textwidth}
    \centering
    \includegraphics[width=\textwidth]{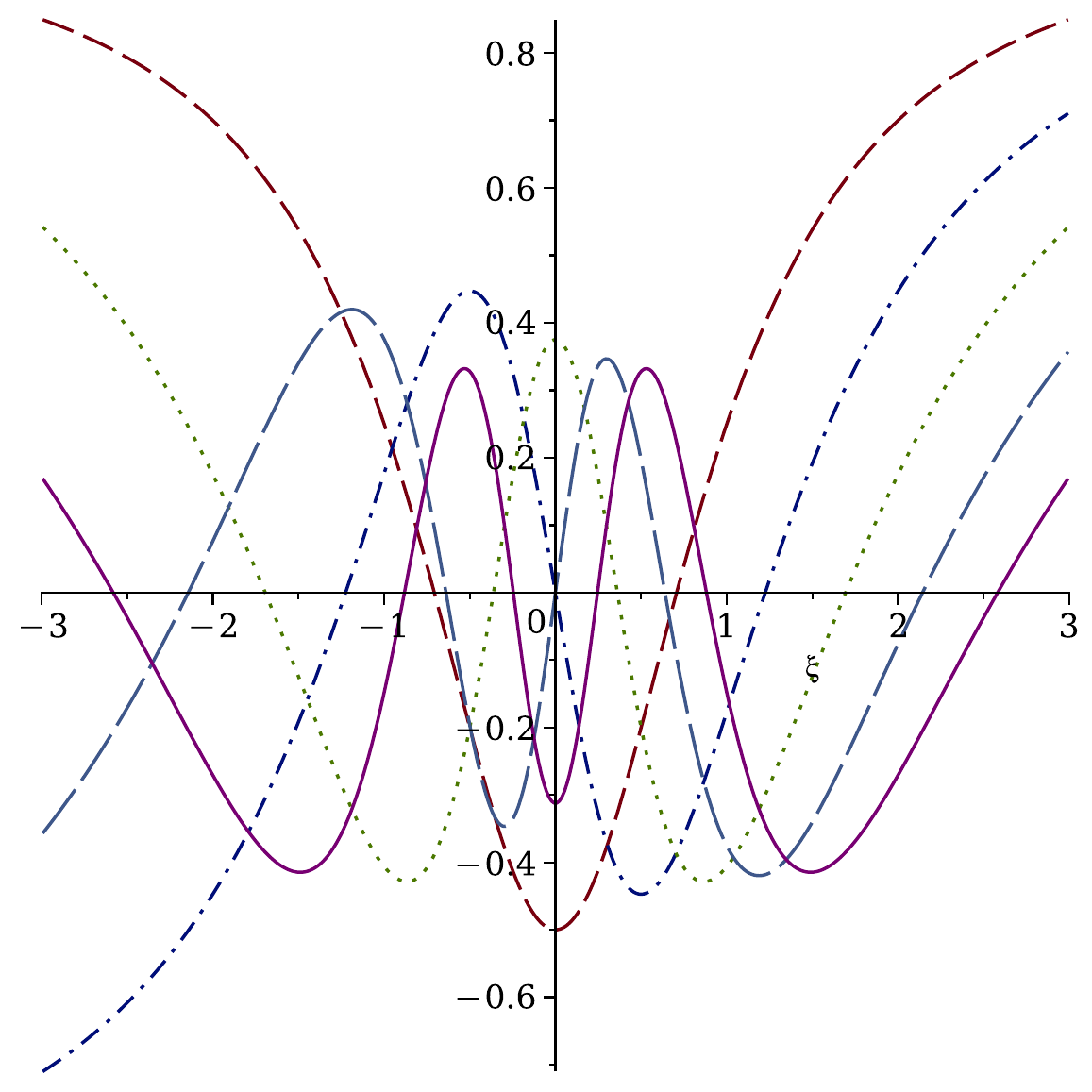}
    \caption{Algebraic mapping $(-\infty, \infty)$}
    \label{fig:alg-inf-jacobi}
  \end{subfigure}
  \caption{Plots of mapped Jacobi functions $\mathcal{J}_n^{(0,0)}(\xi)$ for $n=2,3,\ldots,6$ over finite, semi-infinite, and infinite domains.}
  \label{fig:jacobi-plots}
\end{figure}

\section{Rational KAN}
In this section, we introduce two approaches for developing rKANs. To begin, we briefly review the original KANs by stating the Kolmogorov-Arnold theorem.

\begin{theorem}
For any continuous function \( F : [0,1]^\nu \to \mathbb{R} \), there exist continuous functions \( \varphi_{q,k} : [0,1] \to \mathbb{R} \) and continuous functions \( \psi_k : \mathbb{R} \to \mathbb{R} \) such that
\[
F(\xi_1, \xi_2, \ldots, \xi_\nu) = \sum_{k=1}^{2\nu+1} \psi_k \left( \sum_{q=1}^{\nu} \phi_{q,k}(\xi_q) \right).
\]
\end{theorem}

\begin{proof}
The proof of the Kolmogorov-Arnold representation theorem is highly non-trivial and relies on advanced concepts in functional analysis. For more detailed definitions and proofs, refer to \cite{liu2024kan, braun2009constructive, dzhenzher2021structured, schmidt2021kolmogorov}.
\end{proof}
Employing this theorem, KANs suggest using a nested combination of this approximation for more accurate predictions. In matrix form, KANs are defined as:
\begin{equation*}
    \hat{F}(\boldsymbol\xi)={\bf \Phi}_{L-1}\circ\cdots \circ{\bf \Phi}_1\circ{\bf \Phi}_0\circ {\boldsymbol{\xi}},
\end{equation*}
where ${\bf \Phi}_{q,k} = \phi_{q,k}(\cdot)$ and $\boldsymbol{\xi}$ is the input sample of the network.

To employ a rational basis function in this approach, one can define the functions $\phi_{q,k}(\cdot)$ using a rational function. There are two approaches to generate such a rational function. The first approach is to divide two polynomials, known as the Padé approximation. The second approach is to use a rationalized form of Jacobi functions. In the following, we explain these two approaches.

\subsection{Padé approximation}
The Padé approximation is a method for approximating a function by a rational function of the given order. Specifically, a Padé approximant of order \([q/k]\) for a function \( F(\xi) \) is:
\[
F^{[q/k]}(\xi) \approx\frac{A_q(\xi)}{B_k(\xi)}= \frac{\displaystyle\sum_{i=0}^{q} a_i \xi^i}{\displaystyle\sum_{j=0}^{k} b_j \xi^j},
\]
where \( a_i, b_i \) are real-valued numbers. The polynomials \( A_q(\xi) \) and \( B_k(\xi) \) can be chosen as the original or finite-shifted Jacobi polynomials. For rKAN, we consider the functions \( \phi_{q,k} \) as:
\[
\phi_{q,k}(\boldsymbol\xi) = \frac{\displaystyle\sum_{i=0}^{k}\theta^e_i\mathcal{R}^{(\alpha,\beta)}_i(\boldsymbol\xi_q)}{\displaystyle\sum_{i=0}^{p}\theta^d_i\mathcal{R}^{(\alpha,\beta)}_i(\boldsymbol\xi_q)}.
\]
Here \( \theta^e_i \) and \( \theta^e_i \) are trainable weights and $p$ is a positive integer. Note that the input of the Jacobi polynomial should lie within a specific domain \( [d_0, d_1] \); therefore, a bounded range activation function such as Sigmoid or hyperbolic tangent (namely $\sigma(\cdot)$) should be applied to the input of these functions. Finally, the Padé-rKAN is defined as:
\[
F(\boldsymbol{\xi}) = \sum_{k=1}^{K} \psi_k \left( \sum_{q=1}^{\nu} \phi_{q,k}(\sigma(\boldsymbol\xi_q)) \right),
\]
in which the functions \( \psi_k(\cdot) \) are considered as linear functions. In this formulation, the fractional rational KAN (frKAN) is applicable if we use a linear mapping function on Jacobi polynomials that shifts data to the positive part of the real line. Suppose we use \( \varphi(\xi) = 2 \xi^\gamma - 1 \) with the Sigmoid function \( \sigma(\cdot) \) and a trainable positive fractional order parameter \( \gamma \). The fractional rational basis functions then take the form:
\[
\phi_{q,k}(\boldsymbol\xi) = \frac{\displaystyle\sum_{i=0}^{q}\theta^e_i\mathcal{R}^{(\alpha,\beta)}_i(\varphi(\sigma(\boldsymbol\xi)))}{\displaystyle\sum_{i=0}^{p}\theta^d_i\mathcal{R}^{(\alpha,\beta)}_i(\varphi(\sigma(\boldsymbol\xi)))}.
\]

\subsection{Rational Jacobi functions}
Another approach to using a rational function in KANs is to map the Jacobi polynomials using a nonlinear rational mapping. These mappings can be defined on a semi-infinite domain or on the entire real line. Since the output of a network layer is unbounded, using an infinite mapping such as \eqref{eq:map_alg2} or \eqref{eq:map_log2} can be beneficial. As a result, the basis functions \( \varphi_{q,k}(\cdot) \) are defined as:
\begin{equation*}
    \varphi_{q,k}(\boldsymbol\xi_q) = \mathcal{J}_k^{(\alpha, \beta)}\left(\varphi(\boldsymbol\xi_k; \text{SoftPlus}(\iota))\right),
\end{equation*}
where the soft plus function is defined as:
\begin{equation*}
\begin{aligned}
&\text{SoftPlus}: \mathbb{R}\to (0,\infty),\\
&\text{SoftPlus}(\iota) = \log(1+\exp(\iota)),
\end{aligned}
\end{equation*}
and is applied to the trainable parameter $\iota$ to ensure its positiveness. Similar to the Padé-rKAN, the approximation of Jacobi-rKAN takes the form:
\[
F(\boldsymbol{\xi}) = \sum_{k=1}^{K} \psi_k \left( \sum_{q=1}^{\nu} \phi_{q,k}(\sigma(\boldsymbol\xi_q)) \right).
\]
In this case, to apply fractional basis functions, it is necessary to use a positive range function \( \sigma \) along with a rational mapping that is defined for positive values. Suitable rational mappings include those given by formulas such as \eqref{eq:map_alg} (algebraic mapping), \eqref{eq:map_exp} (exponential mapping), and \eqref{eq:map_log} (logarithmic mapping). The other definitions within the framework remain unchanged.

\section{Experiments}
In this section, we evaluate the proposed rational KAN on various deep learning tasks. All experiments are implemented in Python using the PyTorch and TensorFlow libraries. The experiments are conducted on a PC equipped with an Intel Core i3-10100 CPU, an Nvidia GeForce GTX 1650 GPU, and 16GB of RAM. The implementation of this approach is publicly available on GitHub\footnote{\url{https://github.com/alirezaafzalaghaei/rKAN}}.

\subsection{Deep learning}
This section presents classification and regression tasks simulated using rKAN.
\subsubsection{Regression Tasks}
We begin the assessment of rKAN with a regression task using synthetic data generated from three different functions with asymptotic behavior. These functions are defined as follows and are illustrated in Figure \ref{fig:regression}:

\[
\begin{aligned}
    F_1(\xi) &= \frac{\xi}{1+\xi^2},\\
    F_2(\xi) &= \frac{1}{1+\xi^2},\\
    F_3(\xi) &= \exp(-\xi^2).
\end{aligned}
\]

For training, we sample 200 random points and for testing, 100 random points within the interval \([-10, 10]\). We use a neural network with an architecture of [1, 10, 1], where the hidden layer contains 10 neurons. The network is optimized using the L-BFGS optimizer with full batch processing for 50 epochs. We use the mean squared error (MSE) to evaluate both the training and testing accuracy. The MSE results for the test data are presented in Tables \ref{tbl:f1}, \ref{tbl:f2}, and \ref{tbl:f3}. In all tables, we have employed \eqref{eq:map_alg2} as the rational mapping of Jacobi functions.

\begin{figure}[ht]
    \centering
    \includegraphics[width=0.8\textwidth]{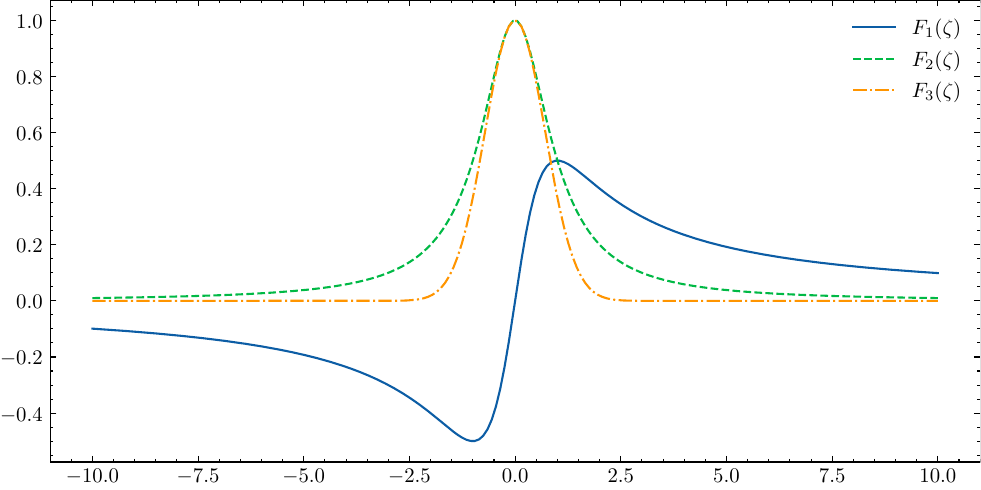}
    \caption{The plots of the functions \( F_1(\xi) \), \( F_2(\xi) \), and \( F_3(\xi) \). The prediction results of rKAN, fKAN, and KAN are presented in Tables \ref{tbl:f1}, \ref{tbl:f2}, and \ref{tbl:f3}.}
    \label{fig:regression}
\end{figure}
\begin{table}[ht]
\centering
\resizebox{1\textwidth}{!}{%
\begin{tabular}{@{}lcccccc@{}}
\toprule
Model & K=2 & K=3 & K=4 & K=5 & K=6  \\ \midrule
fKAN & \num{3.330e-07} & \num{7.454e-07} & \num{6.100e-07} & \num{3.339e-06} & \num{6.967e-06}   \\
Jacobi-rKAN & \num{4.223e-07} & \num{4.289e-07} & \num{3.273e-06} & \num{4.861e-05} & \num{1.132e-04}   \\
Padé[K/3]-rKAN & \num{2.616e-07} & \num{2.778e-02} & \num{1.634e-07} & \num{7.760e-07} & \num{2.142e-06}   \\
Padé[K/4]-rKAN & \num{2.951e-05} & \num{1.075e-06} & \num{1.109e-06} & \num{9.297e-07} & \num{1.722e-05}   \\
Padé[K/5]-rKAN & \num{4.334e-03} & \num{3.930e-03} & \num{1.042e-04} & \num{1.140e-03} & \num{9.587e-04}   \\
Padé[K/6]-rKAN & - & - & \num{7.672e-04} & - & - &  \\ \midrule
Tanh & \num{2.711e-07} & ReLU & \num{5.143e-04} & KAN & \num{2.240e-02} \\ \bottomrule
\end{tabular}
}
\caption{The MSE between the predicted values and the exact values for the \( F_1(\xi) \).}
\label{tbl:f1}
\end{table}

\begin{table}[ht]
\centering
\resizebox{1\textwidth}{!}{%
\begin{tabular}{@{}lcccccc@{}}
\toprule
Model & K=2 & K=3 & K=4 & K=5 & K=6 \\ \midrule
fKAN & \num{5.406e-07} & \num{4.629e-07} & \num{2.170e-06} & \num{3.699e-06} & \num{4.684e-06}   \\
Jacobi-rKAN(K) & \num{1.889e-07} & \num{8.220e-07} & \num{1.446e-06} & \num{9.917e-06} & \num{7.252e-05}   \\
Padé[K/3]-rKAN & \num{2.560e-07} & \num{3.752e-07} & \num{7.343e-07} & \num{7.138e-07} & \num{2.728e-06}   \\
Padé[K/4]-rKAN & \num{9.726e-04} & \num{6.778e-07} & \num{1.521e-06} & \num{1.727e-06} & \num{6.301e-04}   \\
Padé[K/5]-rKAN & \num{4.420e-03} & \num{3.629e-03} & \num{6.551e-04} & \num{1.813e-03} & \num{2.964e-04}   \\
Padé[K/6]-rKAN & - & - & \num{3.364e-04} & \num{1.279e-06} & \num{4.787e-02} &  \\ \midrule
Tanh &  \num{2.963e-07} & ReLU &  \num{3.402e-05} & KAN & \num{1.520e-02} \\ \bottomrule
\end{tabular}
}
\caption{The MSE between the predicted values and the exact values for the \( F_2(\xi) \).}
\label{tbl:f2}
\end{table}

\begin{table}[ht]
\centering
\resizebox{1\textwidth}{!}{%
\begin{tabular}{@{}lcccccc@{}}
\toprule
Model & K=2 & K=3 & K=4 & K=5 & K=6  \\ \midrule
fKAN & \num{4.320e-07} & \num{4.940e-07} & \num{7.540e-07} & \num{4.100e-06} & \num{1.420e-05}   \\
Jacobi-rKAN(K) & \num{5.330e-07} & \num{5.350e-07} & \num{3.540e-06} & \num{1.730e-05} & \num{2.760e-04}   \\
Padé[K/3]-rKAN & \num{2.590e-07} & \num{1.100e-02} & \num{5.670e-07} & \num{7.750e-07} & \num{2.680e-06}   \\
Padé[K/4]-rKAN & \num{1.460e-06} & \num{5.250e-07} & \num{6.990e-07} & \num{9.890e-02} & \num{4.550e-05}   \\
Padé[K/5]-rKAN & \num{4.070e-03} & \num{4.420e-03} & \num{4.340e-06} & \num{1.290e-05} & \num{1.240e-03}   \\
Padé[K/6]-rKAN & \num{4.420e-02} & \num{9.670e-04} & \num{3.790e-06} & - & \num{1.750e-05} &  \\ \midrule
Tanh & \num{1.490e-06} & ReLU & \num{4.750e-05} & KAN & \num{1.890e-02} \\ \bottomrule
\end{tabular}
}
\caption{The MSE between the predicted values and the exact values for the \( F_3(\xi) \).}
\label{tbl:f3}
\end{table}

\subsubsection{MNIST classification}
Recent studies have explored the application of KANs in various image processing tasks, such as image classification \cite{azam2024suitability, seydi2024unveiling, aghaei2024fkan, cheon2024kolmogorov}, image denoising \cite{aghaei2024fkan} image segmentation \cite{li2024ukan}. In this section, we focus on the classification task using the MNIST dataset, which includes 60,000 training images of handwritten digits and 10,000 test images, each with a size of $28 \times 28 \times 1$. We designed a 2-dimensional convolutional neural network for this task, as illustrated in Figure \ref{fig:mnist-architecture}. The network was trained using the Adam optimizer with the default learning rate in Keras, a batch size of 512, and 30 epochs. The validation loss and accuracy during training are depicted in Figure \ref{fig:mnist-training}. Additionally, the performance metrics on the test set, including accuracy and loss, are presented in Table \ref{tbl:mnist}, which also compares our results with those obtained using fractional KAN \cite{aghaei2024fkan} and common activation functions like hyperbolic tangent and ReLU.

\begin{figure}[ht]
        \centering
        \includegraphics[width=1\textwidth]{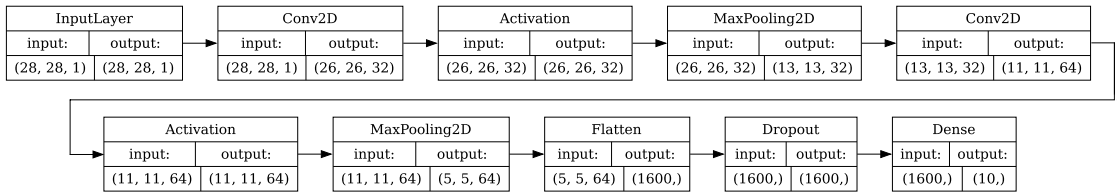}
        \caption{The architecture of proposed method for MNIST classification data.}
    \label{fig:mnist-architecture}
\end{figure}

\begin{figure}[ht]
        \centering
        \includegraphics[width=0.9\textwidth]{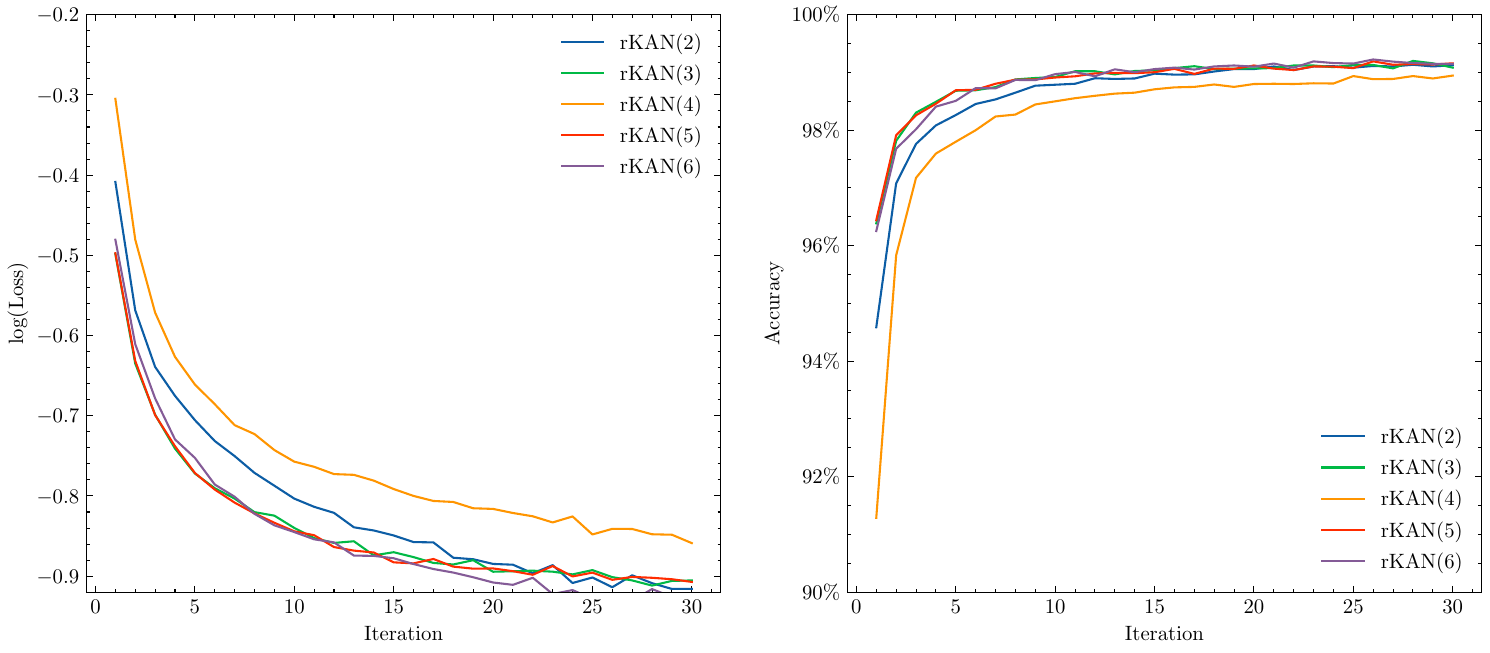}
        \caption{Loss and accuracy of MNIST classification using Jacobi-rKAN with different values of K.}
    \label{fig:mnist-training}
\end{figure}

\begin{table}[ht!]
\centering

\renewcommand{\tabularxcolumn}[1]{m{#1}}
\begin{tabularx}{0.9\textwidth}{@{}*{5}{X}@{}}

    \toprule
    Act. Func.  & \multicolumn{2}{c}{Loss} & \multicolumn{2}{c}{Accuracy} \\ \midrule
                & Mean        & Std.       & Mean          & Std.         \\ \cmidrule(lr){2-3} \cmidrule(lr){4-5}
    Sigmoid     & $0.0611$    & $0.0028$   & $98.092$      & $0.0937$     \\
    Tanh        & $0.0322$    & $0.0015$   & $98.904$      & $0.0695$     \\
    ReLU        & $0.0256$    & $0.0010$   & $99.140$      & $0.0434$     \\
    fKAN(2) \cite{aghaei2024fkan}     & $0.0252$    & $0.0017$   & $99.134$      & $0.0484$     \\
    fKAN(3) \cite{aghaei2024fkan}    & $0.0224$    & $0.0019$   & $99.200$      & $0.0787$     \\
    fKAN(4 )\cite{aghaei2024fkan}     & $0.0217$    & $0.0008$   & $99.228$      & $0.0515$     \\
    fKAN(5) \cite{aghaei2024fkan}     & $0.0249$    & $0.0009$   & $99.204$      & $0.0467$     \\
    fKAN(6) \cite{aghaei2024fkan}     & $0.0290$    & $0.0028$   & $99.024$      & $0.1198$     \\
    rKAN(2)     & $0.0215$    & $0.0012$   & $99.268$      & $0.0683$     \\
    rKAN(3)     & $0.0222$    & $0.0012$   & $99.210$      & $0.0464$     \\
    rKAN(4)     & $0.0292$    & $0.0006$   & $99.060$      & $0.0332$     \\
    rKAN(5)     & $\mathbf{0.0213}$    & $0.0004$   & $\mathbf{99.293}$      & $0.0597$     \\
    rKAN(6)     & $0.0214$    & $0.0027$   & $99.218$      & $0.0944$     \\ \bottomrule
    \end{tabularx}
\caption{Performance of different activation functions in a CNN for classifying MNIST dataset. It is observed that rKAN outperforms fKAN in certain cases.}
\label{tbl:mnist}
\end{table}

\subsection{Physics-informed Deep Learning}

Physics-informed deep learning tasks often involve mathematical problems augmented with real-world data, providing researchers with more precise insights into both the data and the governing equations. KANs have been developed to address these challenges \cite{aghaei2024fkan, liu2024kan,nehma2024leveraging, wang2024kolmogorov, shukla2024comprehensive}. In these networks, the loss function is defined to enable the network to approximate the dynamics of physical problems. For example, for a differential equation in the operator form $\mathcal{L}(F) = 0$ with initial condition $F(0)=F_0$, the network loss is defined as the mean squared residual \cite{firoozsalari2023deepfdenet}:
\begin{equation*}
    \text{Loss}(\boldsymbol\xi) = \frac{1}{|\boldsymbol{\xi}|} \sum_{i=1}^{|\boldsymbol\xi|}{\mathcal{L}(F)(\boldsymbol\xi_i)^2} + |\hat{F}(0) - F_0|^2,
\end{equation*}
where $\boldsymbol{\xi}$ represents the training data in the domain of the problem. In this section, we evaluate two examples of data-driven solutions to differential equations using rKAN.

\subsubsection{Ordinary Differential Equations}
For this task, we will focus on the Lane-Emden equation, a well-known ordinary differential equation. This equation represents a dimensionless form of Poisson's equation, which describes the gravitational potential of a Newtonian, self-gravitating, spherically symmetric, polytropic fluid. This equation, for a positive integer $w$, is defined as follows:

\begin{equation*}
    \begin{aligned}
    \frac{\text{d}^2}{\text{d}\xi^2}F(\xi) &+ \frac2\xi \frac{\text{d}}{\text{d}\xi}F(\xi) + F^w(\xi) = 0,\\
    F(0)&=1,\quad   F'(0)=0.
\end{aligned}
\end{equation*}
To simulate this problem, we use the rKAN architecture with Padé and Jacobi rational mapping \eqref{eq:map_alg2}. For a fair comparison, we adopt a network architecture similar to that of fKAN \cite{aghaei2024fkan}, but replace the fKAN layers with rKAN layers. This network incorporates six different Jacobi basis functions (i.e., \(K=6\) in our rKAN architecture) and is optimized using the L-BFGS algorithm with 1500 equidistant points in the domain $[0,15]$.

The first roots of the predicted solution to the differential equation hold significant physical meaning. Therefore, in Table \ref{tbl:lane-emden}, we compare our results with those obtained from fractional KAN \cite{aghaei2024fkan} and the Grammatical Evolution Physics-Informed Neural Network (GEPINN) \cite{mazraeh2024gepinn}.

\begin{table}[ht]
\centering
\begin{tabular}{@{}ccccc@{}}
\toprule
$w$ &        Jacobi-rKAN       &    Padé[q/6]-rKAN        & fKAN \cite{aghaei2024fkan} & GEPINN \cite{mazraeh2024gepinn} \\ \midrule
0   &  $5.15 \times 10^{-6}$   &  $4.86 \times 10^{-6}$   & $3.52 \times 10^{-5}$     & $1.40 \times 10^{-7}$ \\
1   &  $7.12 \times 10^{-5}$   &  $8.67 \times 10^{-6}$   & $8.67 \times 10^{-6}$     & $4.83 \times 10^{-3}$ \\
2   &  $2.88 \times 10^{-5}$   &  $5.09 \times 10^{-5}$   & $9.34 \times 10^{-6}$     & $8.93 \times 10^{-3}$ \\
3   &  $2.40 \times 10^{-4}$   &  $1.06 \times 10^{-5}$   & $5.55 \times 10^{-7}$     & $1.88 \times 10^{-2}$ \\
4   &  $1.57 \times 10^{-3}$   &  $2.82 \times 10^{-2}$   & $2.97 \times 10^{-4}$     & $5.08 \times 10^{-2}$ \\ \bottomrule
\end{tabular}
\caption{Comparison of the first roots of the predicted solution with the exact roots from \cite{horedt2004polytropes} and the approximated results from a similar neural network approaches \cite{mazraeh2024gepinn, aghaei2024fkan}.
}
\label{tbl:lane-emden}
\end{table}

\subsubsection{Partial Differential Equations}
For a more challenging task, we select an elliptic partial differential equation (PDE) defined as follows:
\begin{equation}
\begin{aligned}
\frac{\partial^2}{\partial \xi_1^2} F(\xi_1, \xi_2) &+\frac{\partial^2}{\partial \xi_2^2} F(\xi_1, \xi_2)  = \sin(\pi \xi_1) \sin(\pi\xi_2), \\
F(\xi_1,0) &= 0,\quad F(\xi_1,1) = 0,\\
F(0,\xi_2) &= 0,\quad F(1,\xi_2) = 0.
\end{aligned}
\label{eq:pde}
\end{equation}
The exact solution to this PDE is given by \cite{mall2017single}:
\begin{equation*}
    F\left(\xi_1,\xi_2\right) = -\frac{1}{2\pi^{2}} \sin\left(\pi \xi_1\right) \sin\left(\pi \xi_2\right).
\end{equation*}

We simulate the solution of this PDE using a simple rKAN with the architecture $[1, 10, 10, 1]$ and Jacobi-rKAN basis functions of order 4 using $50\times 50$ datapoints in $[0,1]^2$. The simulation results for this problem are shown in Figure \ref{fig:pde}.

\begin{figure}[htbp]
  \centering
  \begin{subfigure}{0.4\textwidth}
    \centering
    \includegraphics[width=\textwidth]{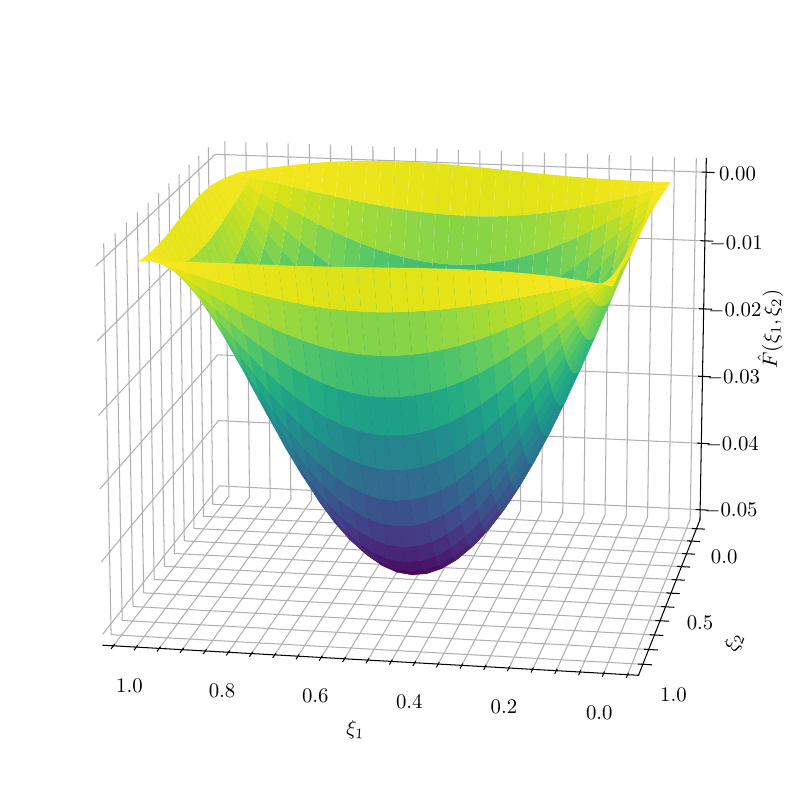}
    \caption{Prediction}
    
  \end{subfigure}%
  \begin{subfigure}{0.4\textwidth}
    \centering
    \includegraphics[width=\textwidth]{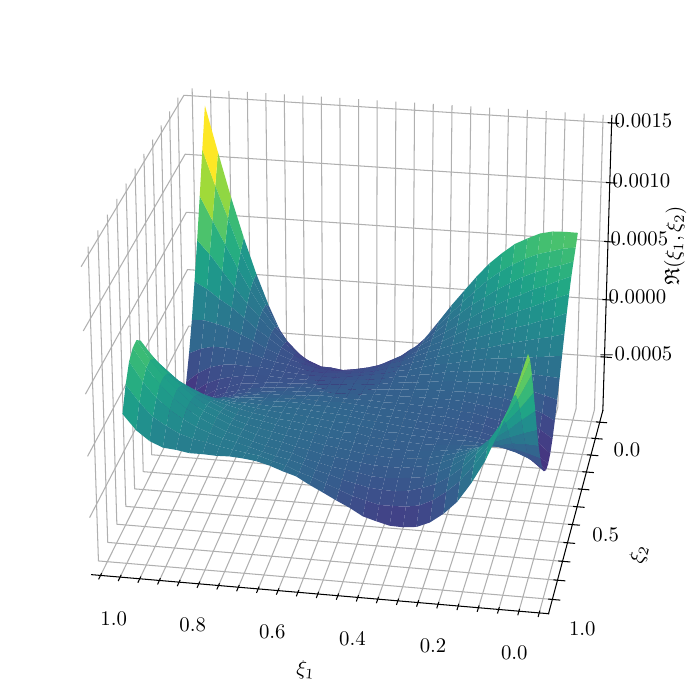}
    \caption{Residual}
    
  \end{subfigure}
  \caption{The predicted solution and the residual function with respect to the exact solution for the elliptic PDE given in Equation \eqref{eq:pde}.}
  \label{fig:pde}
\end{figure}
\section{Conclusion}
In this paper, we have introduced a new perspective on Kolmogorov-Arnold networks utilizing rational basis functions. Rational functions, a type of basis function in numerical approximation, enhance prediction accuracy, particularly in scenarios involving asymptotic behavior and singularities. We proposed two types of rational KANs based on the Padé approximation and rational Jacobi functions. The first architecture employs the division of two polynomials, specifically shifted Jacobi functions, while the second approach maps Jacobi functions directly into a rational space. In both models, the basis function hyperparameters $\alpha$, $\beta$, and $\iota$ (for rational Jacobi functions) are optimized as network weights. We also demonstrated that our method can be integrated with fractional KANs in certain contexts.

We validated the effectiveness of the proposed method through simulations on real-world examples, including a regression task, a classification task, and numerical approximations for solving the Lane-Emden ordinary differential equation and an elliptic partial differential equation. The results indicate that our method can sometimes achieve greater accuracy compared to existing alternatives. However, our experiments showed that the Padé-rKAN increases the time complexity of training, as it involves the computation of two weighted polynomials.

For future work, we suggest exploring the use of rational versions of B-spline curves \cite{tiller1983rational, bardis1990surface}, which are renowned for their flexible shape representation capabilities. Furthermore, a focused evaluation of fractional rational KANs is warranted, particularly for solving physics-informed problems defined on semi-infinite domains \cite{babaei2024solving, aghaei2023solving}.

\printbibliography

\end{document}